\documentclass[sigconf]{acmart}
\usepackage[linesnumbered,ruled]{algorithm2e}
\usepackage{subcaption}
\usepackage{diagbox}
\usepackage{todonotes}
\usepackage{array}
\usepackage{balance} 

\usepackage{amsthm}

\newtheoremstyle{mystyle}
  {}
  {}
  {\itshape}
  {}
  {\bfseries}
  {.}
  { }
  {}

\theoremstyle{mystyle}

\newtheorem{theorem}{Theorem}[section]
\newtheorem{lemma}{Lemma}[section]
\newtheorem{corollary}{Corollary}[lemma]
\newtheorem{definition}{Definition}[section]

\def\BibTeX{{\rm B\kern-.05em{\sc i\kern-.025em b}\kern-.08emT\kern-.1667em\lower.7ex\hbox{E}\kern-.125emX}}
    
\balance 

\begin{document}

\copyrightyear{2019}
\acmYear{2019}
\setcopyright{acmcopyright}
\acmConference[KDD '19]{The 25th ACM SIGKDD Conference on Knowledge Discovery and Data Mining}{August 4--8, 2019}{Anchorage, AK, USA}
\acmBooktitle{The 25th ACM SIGKDD Conference on Knowledge Discovery and Data Mining (KDD '19), August 4--8, 2019, Anchorage, AK, USA}
\acmPrice{15.00}
\acmDOI{10.1145/3292500.3330926}
\acmISBN{978-1-4503-6201-6/19/08}

\settopmatter{printacmref=true}
\fancyhead{}

%
\title{Enhancing Domain Word Embedding via Latent Semantic Imputation}

%

\author{Shibo Yao}
\affiliation{%
 \institution{New Jersey Institute of Technology}
 \streetaddress{323 Dr Martin Luther King Jr Blvd}
 \city{Newark}
 \country{United States}}
\email{espoyao@gmail.com}

\author{Dantong Yu} 
\authornote{Corresponding Author.}
\affiliation{%
 \institution{New Jersey Institute of Technology}
 \streetaddress{323 Dr Martin Luther King Jr Blvd}
 \city{Newark}
 \country{United States}}
\email{dtyu@njit.edu}

\author{Keli Xiao} 
\affiliation{%
 \institution{Stony Brook University}
 \streetaddress{346 Harriman Hall }
 \city{Stony Brook}
 \country{United States}}
\email{keli.xiao@stonybrook.edu}

%

%
\begin{abstract}
We present a novel method named Latent Semantic Imputation (LSI) to transfer external knowledge into semantic space for enhancing word embedding.   The method integrates graph theory to extract the latent manifold structure of the entities in the affinity space and leverages non-negative least squares with standard simplex constraints and power iteration method to derive spectral embeddings. It provides an effective and efficient approach to combining entity representations defined in different Euclidean spaces.  Specifically, our approach generates and imputes reliable embedding vectors for low-frequency words in the semantic space and benefits downstream language tasks that depend on word embedding. We conduct comprehensive experiments on a carefully designed classification problem and language modeling and demonstrate the superiority of the enhanced embedding via LSI over several well-known benchmark embeddings. We also confirm the consistency of the results under different parameter settings of our method. 

\end{abstract}


%

%
\maketitle

\section{Introduction}

Embedding words and phrases in continuous semantic spaces has a significant impact on many natural language processing (NLP) tasks.   There are two main types of word embedding approaches aiming at mapping words to real-valued vectors: matrix factorization~\citep{deerwester-indexing-1990,lund_producing_1996,pennington2014glove} and neural networks~\cite{bengio2003neural,mikolov2013distributed}. 
The embedding matrix is a reflection of the relative positions of words and phrases in semantic spaces. 
Embedding vectors  pre-trained on large corpora can serve as the embedding layer in neural network models and may significantly benefit downstream language tasks because reliable external information is transferred from large corpora to those specific tasks. 

Word embedding techniques have achieved remarkable successes in recent years. Nonetheless, there are still critical questions that are not well answered in the literature. 
For example, considering that word embedding usually relies on the distributional hypothesis~\citep{harris1954distributional}, it is difficult to generate reliable word embeddings if the corpus size is small or if indispensable words have relatively low frequencies~\citep{bojanowski2016enriching}.
Such cases can happen in various domain-specific language tasks, \textit{e.g.}, chemistry, biology, and healthcare, where thousands of domain-specific terminologies exist.  
This challenge naturally drives us to find an effective way to leverage available useful information sources to enhance the embedding vectors of words and phrases in domain-specific NLP tasks.  From the perspective of linear algebra,  the problem falls into the scope of combining two matrices, both of which describe entities in their own spaces. 
Theoretically, the combination task can be addressed even if the two matrices do not have identical dimensions; however, the sets of entities they describe must have an intersection to be combined. 

In this paper, we propose a novel method, named \textit{Latent Semantic Imputation}, to combine representation matrices, even those representation matrices defined in different Euclidean spaces. The framework can be further adapted to address the problem of unreliable word embedding and to effectively and efficiently transfer domain affinity knowledge to semantic spaces or even from one semantic space to another. The knowledge transfer becomes the indispensable strategy to fundamental NLP tasks and  ultimately improves their performance.
In summary, the contributions of this paper are as follows:
\begin{itemize}
    \item We formalize the problem of combining entity representations in different Euclidean spaces and propose the method of Latent Semantic Imputation (LSI\footnote{Code available at https://github.com/ShiboYao/LatentSemanticImputation}) to solve the problem. 
    \item To ensure the deterministic convergence, we propose a minimum spanning tree $k$-nearest neighbor graph (MST-$k$-NN) and seek a sparse weight matrix via solving non-negative least squares problems with standard simplex constraints. Such a sparse weight matrix is a representation of the latent manifold structure of the data. 
    \item We design a power iteration method that depicts a random walk on a graph to derive the unknown word embedding vectors while preserving the existing embeddings.  
    \item We demonstrate how LSI can be further adapted for enhancing word embedding and will benefit downstream language tasks through experiments. 
\end{itemize}

\section{Related Work}
A typical matrix factorization approach for word embedding is based on word co-occurrence matrices. Such a matrix implicitly defines a graph in which the vertices represent words and the matrix itself encodes the affinity.
By taking a set of significant eigenvectors of this matrix, we attain low-dimensional word embedding, with a couple of variations~\citep{bullinaria2007extracting,turney2010frequency,deerwester-indexing-1990,hashimoto2016word,levy2014neural,lund_producing_1996}. 
The neural network approach for word embedding dates back to the neural probabilistic language model~\citep{bengio2003neural}, where the embeddings are precisely the weight matrix of the first hidden layer and generated during the language model training process.  The common practice nowadays is to put the embeddings pre-trained from a large corpus back into the embedding layer of the neural network and facilitate the training of language tasks. 
These two approaches appear to be isolated from each other; nevertheless, they have surprisingly fundamental connections. Levy and Goldberg illustrated that skip-gram with negative sampling (SGNS)~\citep{mikolov2013distributed} implicitly factorizes the word-context matrix constructed from the point-wise mutual information~\citep{levy2014neural}.  
\citet{li2015word} argued that SGNS is an explicit factorization of the word co-occurrence matrices.  From the perspective of vector space models of semantics~\citep{turney2010frequency}, the log co-occurrence of words is related to their semantic similarity, and the word embedding method is a metric recovery process in semantic spaces with manifold learning~\citep{hashimoto2016word}. 

Our study is closely related to spectral embedding methods that try to calculate the most significant eigenvectors (or linear combinations of the significant eigenvectors) of the adjacency matrix.
The idea of spectral embedding methods dates back to principal component analysis~\citep{diamantaras1996principal} that is widely used for linear dimensionality reduction. With the Laplacian matrix, spectral embedding methods capture the nonlinearity and expand its application cases into clustering~\citep{von2007tutorial,lin2010power,huang2014diverse} and nonlinear dimensionality reduction. 

The construction of the weight matrix (or affinity matrix) is essential in spectral embedding.  The weight matrices are usually sparse due to the assumption that data points in Euclidean spaces only ``interact'' with their nearest neighbors and reveal the Euclidean geometrical properties locally~\citep{saul2003think}. 
In other words, the high-dimensional data usually follow a lower-dimensional manifold distribution~\citep{roweis2000nonlinear}. 
Manifold learning has a wide range of applications in representation learning and data visualization~\citep{belkin2003laplacian, maaten2008visualizing}.
Manifold learning has an intrinsic relationship with graph theory due to the common practice of using graphs to describe the manifold structure embedded in the original high-dimensional space~\citep{roweis2000nonlinear,LiuHC10}. 

Manifold learning models often produce closed-form solutions. However, closed-form solutions usually involve the inverse or eigendecomposition of a large matrix, both of which are computationally expensive. Instead, we leverage the power iteration method to obtain solutions efficiently in the case where the primary concern is the time complexity.  This work is also inspired by transductive learning and label propagation~\citep{zhu2002learning,bengio200611}.  The major difference is that label propagation tries to propagate labels in the label space, while our work seeks to learn representations in the semantic space. For the power iteration to converge deterministically, that is, where the final result does not depend on the initialization of the unknown embeddings, we have to rigorously formulate the processes of constructing the graph and solving its weight matrix.

\section{Problem Definition and Proposed Approach Overview}
\begin{figure}[b]
  \centering
  \includegraphics[width=.9\linewidth]{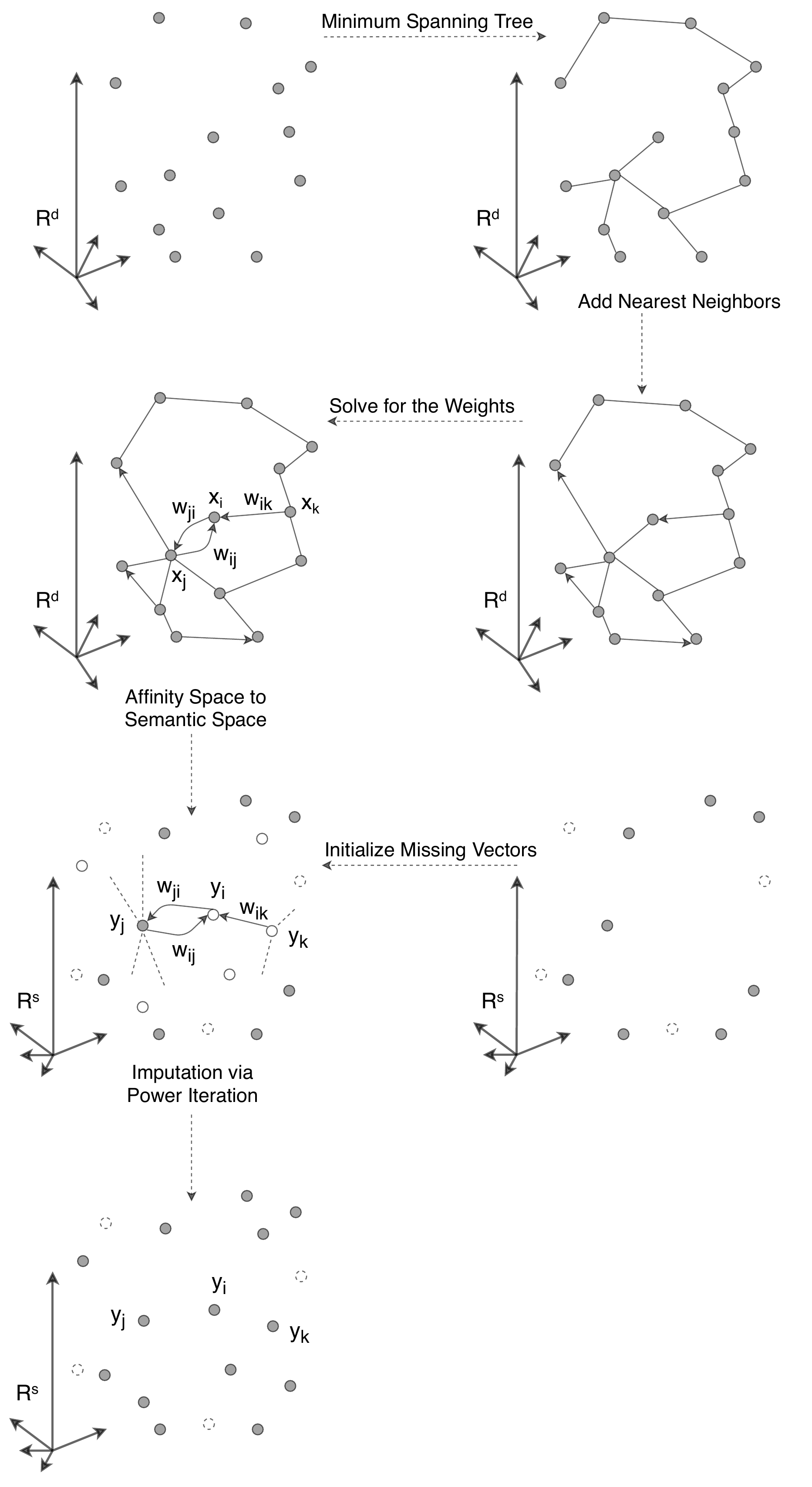}
  \caption{Latent Semantic Imputation. $\mathbb{R}^d$ is the affinity space and $\mathbb{R}^s$ is the semantic space. The grey solid circles in $\mathbb{R}^d$ are the overlapping entities in the two spaces. The hollow circles are the vectors to be imputed. The dotted circles are the embedding vectors already known in the semantic space but irrelevant.}\label{fig:LSI}
  \Description{$\mathbb{R}^d$ stands for $d$-dimensional affinity space, $\mathbb{R}^s$ stands for $s$-dimensional semantic space. }
\end{figure}
The state-of-the-art word embedding methods adopt a common practice of neglecting low-frequency words before training in order to minimize the effect of low-frequency words and ensure reliable results \cite{mikolov2013distributed}\cite{pennington2014glove}.   However, neither excluding these words with $\langle unk \rangle$ nor indiscriminately including them in the vocabulary works well because the valuable information of low-frequency words should be preserved for domain-specific language tasks.  For example, thousands of company names in financial news analysis, disease and medicine abbreviations in health-care problems, and functional and radical groups in chemistry. Sometimes, it is also desired to expand the vocabulary of word embedding to minimize the effect of out-of-vocabulary words. To quantify such problems, we may combine different embedding matrices or fuse data from multiple sources and transfer prior knowledge into the semantic space to enhance the word embedding. 

The underlying problem is to transfer entity representations from one Euclidean space to another. Suppose there is a set of entities $\mathcal{A} = \{a_i | i=0,...,m_1\}$ with their representation vectors in space $\mathbb{R}^{n_1}$ and a set of entities $\mathcal{B} = \{b_i | i=0,...,m_2\}$ whose representation vectors are in space $\mathbb{R}^{n_2}$, where $\mathcal{A} \cap \mathcal{B} \neq \phi$. We would like to approximate the entity representations of $\mathcal{A} \setminus \mathcal{B}$ (entities exist in $\mathcal{A}$ and not in $\mathcal{B}$) in $\mathcal{B}$'s representation space $\mathbb{R}^{n_2}$ or approximate the entity representations of $\mathcal{B} \setminus \mathcal{A}$ in space $\mathbb{R}^{n_1}$. 

To further elaborate the mechanisms of Latent Semantic Imputation,  we take the embedding of company names as an example. Some company names are relatively low-frequency words in corpora. Discarding these names will be a suboptimal choice since the names are likely to be essential for downstream language tasks.  On the other hand, adding them to the dictionary might lead to unreliable word embedding results.  We discovered that different types of information might be complementary to each other.  For example,  we can use the historical trading price of publicly listed companies or the financial features of companies to define another matrix that describes the entity similarity, other than the word embedding matrix.  Thereby, we transfer the rich information from this domain matrix to the word embedding matrix.  From the graph perspective, the graph defined by the domain matrix and the graph defined by the word embedding matrix often have similar topological structures.   Therefore,  we can apply the graph topology that is determined by the domain matrix to recover the local graph structure of the company names that are missing from the semantic space due to their low frequency. The domain matrix is determined by domain statistics and might vary according to the available data source; \textit{e.g.},  a domain matrix in chemistry is based on the co-occurrence of functional or radical groups in a basket of molecules, while  one in healthcare is derived from the medicine or disease features. 

To use  Latent Semantic Imputation to enhance domain word embedding, we take the following steps: (1) Identify a reliable embedding matrix as the base. The embedding is trained with a large corpus, \textit{e.g.}, billions of words or gigabytes of plain text.  (2) Identify or construct a domain matrix that describes the entities within the specific domain. (3) Apply latent semantic imputation to recover those embedding vectors that are not in the initial set of word embedding vectors. 

\section{Latent Semantic Imputation}
We obtained inspiration from manifold learning~\citep{belkin2003laplacian,maaten2008visualizing,saul2003think},  particularly one of the pioneering works, Locally Linear Embedding (LLE)~\citep{roweis2000nonlinear}.  Based on LLE, we argue that the words are in a lower-dimensional manifold embedded in the original high-dimensional space. A manifold is a topological space that locally resembles Euclidean space near each point, and therefore we need to think globally for the overall embedding context and fit locally for each single data point~\citep{saul2003think}.

We specifically denote the domain data matrix  $X \in \mathbb{R}^{n\times d}$, where $n$ is the entity number and $d$ is the dimension, and use $Y_{raw} \in \mathbb{R}^{m\times s}$ to represent the embedding matrix in the semantic space, where $m$ is the vocabulary size and $s$ is the dimension of the semantic space. Note that there are $p$ entities in the intersection of $Y_{raw}$ and $X$. Without loss of generality,   we permute the vectors to arrange the overlapping part to be the upper part of the two matrices.  The bottom $q$ vectors in $X$ will be reconstructed into  $Y_{raw}$. The remaining vectors of $Y_{raw}$ that are outside of the overlapping part are irrelevant. They remain unchanged, do not contribute any information and will be directly fed into downstream language tasks. Hence we temporally discard them.  Our goal is to find the missing $q$ vectors (or matrix $Y_q$) in $Y_{raw}$ given the domain matrix $X$ and the known $p$ vectors ($Y_p$) in the semantic space.  For simplicity, we use $Y \in \mathbb{R}^{n \times s}$ where $n=p+q$ to denote a vertical concatenation of $Y_p$ and $Y_q$, namely, the known embedding vectors and the missing embedding vectors in the semantic space. The whole process is essentially a data imputation. 
\begin{equation}
X=
\begin{bmatrix}
    X_p \\
    X_q \\
\end{bmatrix}
=
\begin {bmatrix}
    x_1^\top \\
    \vdots \\
    x_p^\top \\
    x_{p+1}^\top \\
    \vdots \\
    x_n^\top \\
\end {bmatrix}
\rightarrow
\begin {bmatrix}
    y_1^\top \\
    \vdots \\
    y_p^\top \\
    \hat{y}_{p+1}^\top \\
    \vdots \\
    \hat{y}_n^\top \\
\end {bmatrix}
=
\begin{bmatrix}
    Y_p \\
    Y_q \\
\end{bmatrix}
=
Y
\end{equation}
where $x_i \in \mathbb{R}^d$ and $y_i \in \mathbb{R}^s$ are the vectors in the affinity space and the semantic space, respectively. We denote vectors with bold letters like $\textbf{x}_i$ in the following sections.

Latent Semantic Imputation consists of three steps:\\
\textit{Step 1.} Build a minimum-spanning-tree $k$-nearest neighbor graph given the domain matrix $X$,\\
\textit{Step 2.} Solve for the sparse weight matrix $W$ using non-negative least squares with standard simplex constraints, and \\
\textit{Step 3.} Recover the missing embedding vectors $Y_q$ in the semantic space using the power iteration method.

\subsection{The Minimum-Spanning-Tree $k$-NN Graph}
Locally Linear Embedding argues that a data point is a linear combination of its $k$ nearest neighbors under the manifold assumption~\citep{roweis2000nonlinear}, which means $\textbf{x}_i = \frac{1}{n}\sum_{j=1}^{n}w_{ij} \textbf{x}_j, i\neq j$, where $w_{ij}\neq 0$ if $\textbf{x}_j$ belongs to the $k$ nearest neighbors of $\textbf{x}_i$ and zero otherwise.


Building a $k$-NN graph is the first step in LLE.  However, it is not sufficient in Latent Semantic Imputation. To elaborate on this, we first notice that a $k$-NN graph is not necessarily a connected graph and might contain some connected components with all vertices not having embedding vectors in the semantic space before imputation. In such cases, even if we have the topology derived from the domain matrix, these data points can still be arbitrarily shifted or rotated in the target semantic space because their mean and covariance are not constrained. These embedding results are obviously erroneous. In particular, the idea of Latent Semantic Imputation is to transfer the prior knowledge,  in the format of a domain matrix extracted from other data sources,  into the semantic space and impute those missing embedding vectors. The iterative process requires the graph being connected to propagate information, which is similar to the network flow~\citep{Ahlswede2000}. In a disconnected graph, some of the embedding vectors form data cliques (isolated subgraphs), and the information is not able to propagate through different subgraphs.

On the contrary,  it is desirable that the known embedding vectors serve as ``anchors'' ~\citep{LiuHC10} in the graph. In Figure \ref{fig:LSI}, $\textbf{x}_j$ is a vector in the affinity space, and its corresponding embedding vector in the semantic space is $\textbf{y}_j$.  In the process of imputation,  $\textbf{x}_j$ serves as an anchor to stabilize the positions of unknown word embedding vectors. To this end, we construct a Minimum-Spanning-Tree $k$-Nearest-Neighbor (MST-$k$-NN) graph which ensures the connectivity to describe the topological structure of the entities in the affinity space. We will show in a later section that a MST-$k$-NN graph guarantees deterministic convergence in the power iteration. 
That is, the final results do not depend on the initialization of $Y_q$.  


In the MST-$k$-NN graph construction, we first find a minimum spanning tree~\citep{Kruskal1956} based on the Euclidean distance matrix of entities given the domain matrix $X$. In a minimum spanning tree, some nodes can be dominated by one or two neighbors, which would be likely to introduce bias and sensitivity. Therefore on top of the minimum spanning tree, we search for additional nearest neighbor(s) of each node and add directed edge(s) between these additional neighbor(s) and the corresponding node if the in-degree of the node is smaller than $k$. After this step, the minimum degree of the graph $\delta(\mathcal{G})$ is $k$ (previously, some vertices in the minimum spanning tree can have a degree larger than $k$), and we denote this hyperparameter $\delta$. In practice, we did not find it necessary to  add degree constraints because the graph construction usually preserves the manifold assumption. In other words, the maximum node degree in the minimum spanning tree usually does not exceed $\delta$ or not exceed $\delta$ significantly.

Note that the graph from the minimum spanning tree is un-directed (symmetric adjacency matrix, or two nodes are mutually reachable if they have an edge), while the graph after the nearest neighbor search is directed because the nearest neighbor relationship is not communicative. 
\begin{algorithm}[t]
    \SetKwFunction{Union}{Union}\SetKwFunction{FindCompress}{FindCompress}\SetKwInOut{Input}{Input}\SetKwInOut{Output}{Output}
 
    \Input{($X$, $\delta$) \tcp*{$X$:domain matrix; $\delta$:minimum degree}}
    \Output{$\mathcal{G}=(V,E)$ }
    \BlankLine
    $A = EuclideanDistance(X)$; \\
    $\mathcal{G}=(V,E)\gets Kruskal(A)$;   \\
    \For{$i\gets1$ \KwTo $|V|$}{ 
        $V_i\gets \{v_j\ |\ (v_j, v_i) \notin E \}$; \\ 
        \While{$deg^-(v_i)<\delta$ \tcp*{$deg^-$: in-degree}}{ 
            $v_j = argmin(v_j)\ d(v_i, v_j), v_j \in V_i$; \\
            $E \gets E \cup \{(v_j, v_i)\}$;\\
            $V_i \gets V_i \setminus \{(v_j, v_i)\}$;\\
        }
    }
\caption{MST-$k$-NN Graph}
\end{algorithm}
\subsection{Non-Negative-Least-Squares to Solve $W$}
The second step of Latent Semantic Imputation is to resolve the optimal weights. In this step, we explicitly impose two constraints on the weights: being non-negative and normalized. These two constraints jointly define a standard simplex constraint to the least squares problem. Solving such problems produces a random walk matrix where the weight values are bounded by $0$ and $1$ and all row sums are $1$. Such matrices have the desired property that the absolute eigenvalues have the upper-bound $1$ because the dominant eigenvalue modulus is bounded by the smallest and largest row sums\cite{seneta2006non},  both of which are $1$. The objective function for solving the weight matrix $W$ in the affinity space is 
\begin{equation}
\begin{aligned}
& \underset{\textit{W}}{\text{argmin}}
& & \sum_{i=1}^{n}\left\|\textbf{x}_i - \sum_{j=1}^{n}w_{ij}\textbf{x}_j \right\|^2 \\
& \text{s.t.}
& & \sum_{j=1}^n w_{ij} = 1, i\neq j \\
& & &  w_{ij} \geq 0
\end{aligned}
\label{eqn:linear-combine}
\end{equation}
To solve this problem, note that it can be reduced to $n$ non-negative least squares problems each of which tries to solve for an optimal weight vector with the same constraints, 
$
\underset{\textbf{w}_i}{\text{argmin}}
\left\|\textbf{x}_i - \sum_{j=1}^{n}w_{ij}\textbf{x}_j \right\|^2 ,
$
since solving for weight vector $\textbf{w}_i$ has no influence on solving for $\textbf{w}_j$, $\forall i \neq j$. 

Lawson and Hanson first proposed the active set method to solve non-negative least squares problems~\citep{lawson1995solving}. Once the non-negative weight vectors are obtained~\citep{jones2014scipy}, they will be normalized. Note that the weight matrix $W \in \mathbb{R}^{n\times n}$ is sparse due to the manifold assumption. Furthermore, all column vectors $\textbf{w}_{:,j} \neq \textbf{0}, 1\le j \le n,$ because  the graph generated from the MST step is connected.

\subsection{Power Iteration Method to Solve $Y_q$ Matrix}
Our algorithm rests upon the assumption that the local structures are similar to each other between the affinity space and the semantic space, defined by the sparse weight matrix $W$.   Based on this assumption, we design a mapping that preserves the local structures embedded in the two different spaces,  fixes the known embedding vectors $Y_p$ and lets $Y_q$ adapt to the affinity information.  In this setting, $Y_p$ imposes additional constraints on the optimization problem. Based on the findings in transductive learning and label propagation~\citep{zhu2002learning,bengio200611}, the power iteration method is suitable for the imputation problem.  The final embedding vectors are a series of linear combinations of the dominant eigenvectors of the weight matrix. We denote the lower $q$ rows of the weight matrix as $W_q$ and the time step $t$, and apply the following iterative process to solve for $Y_q$:  
\begin{equation}
    Y_{q}^{(t+1)} = W_{q}Y^{(t)}.
\end{equation}

\subsubsection{Spectral Radius of a Square Matrix}
\label{sec:Spectral}

The spectral radius of a square matrix $M \in \mathbb{R}^{n \times n}$ is 
its largest absolute eigenvalue
$$
\rho(M) = max\{|\lambda_i|\ |i=1,...,n\}.
$$
According to the Perron-Frobenius theorem,  the eigenvalue modulo of $W$ retrieved from step $2$ is bounded by $1$\cite{seneta2006non}. To expand on this, first note that the non-negative weight matrix is irreducible 
because its associated graph is a strongly connected graph\cite{seneta2006non} that is constructed from a minimum spanning tree.  
Therefore,  $\rho(W)$ is bounded by its maximal and minimal row sums,  
\begin{equation}
    \underset{i}{min} \sum_{j=1}^nw_{ij} \leq \rho(W) \leq \underset{i}{max} \sum_{j=1}^nw_{ij}.
\end{equation}
The row sum of $W$ is always $1$ as a result of the normalization in solving $W$. Hence, the spectral radius $\rho(W)=1$. In fact, the dominant eigenvalue is $1$ and unique. This ensures convergence in the power iteration. To fix the known embeddings $Y_p$,  we set $W_p$ to identity, which results in $p$ eigenvalues being $1$ and the rest having a modulo strictly smaller than $1$. 

\subsubsection{Power Method for Eigenvectors}

The power method is often used for calculating the eigenvectors of a given square matrix $M \in \mathbb{R}^{n\times n}$. Given a random 
vector $\textbf{b}^{(0)} \in \mathbb{R}^n$, the power method iteratively multiplies the vector by  matrix $M$ and divides the product by its $l_\infty$ norm $\|\textbf{b}\|_\infty$.  The embedding vector $\textbf{b}^{(t)}$ eventually becomes the dominant eigenvector of matrix $M$. In our case, suppose we have a random vector that is a linear combination of $W$'s eigenvectors (at this moment, $W_P$ is not changed to identity and $W$ has a unique dominant eigenvalue), $\textbf{b}^{(0)} = \sum_{i=1}^n c_i\textbf{v}_i$, where $c_1$ is the weight of the dominant eigenvector and is usually nonzero.
There exists an eigengap between the dominant eigenvalue and the rest, $|\lambda_1| > |\lambda_i|, \forall i > 1$, and consequently, the weights of the non-dominant eigenvectors shrink rapidly and eventually become zero after a certain number of iterations. $\textbf{b}^{(t)}$ eventually becomes the dominant eigenvector $\textbf{v}_1$ of $W$:
$$
\textbf{b}^{(t)} = W^t \textbf{b}^{(0)} = W^t\sum_{i=1}^n c_i\textbf{v}_i = \sum_{i=1}^n c_iW^t\textbf{v}_i = \sum_{i=1}^n c_i\lambda_i^t\textbf{v}_i
$$
and 
$
\lim_{t \rightarrow \infty} c_i\lambda_i^t\textbf{v}_i = \textbf{0}\ \forall i>1,
$
thus,
$
\lim_{t \rightarrow \infty} \textbf{b}^{(t)} = c_1\lambda_1^t\textbf{v}_1 = c_1\textbf{v}_1. 
$
In our case, the initial weights of the dominant eigenvectors are in fact determined by the known embedding vectors.

\subsubsection{Implementation Details}
In practice, $Y_p$ is fixed, and only $Y_q$ needs to be updated during the iteration. Therefore, we set $W_p$ to identity, 
\begin{equation}
\left[
\begin{array}{c|c}
W_{pp} & W_{pq} \\
\hline
W_{qp} & W_{qq}
\end{array}
\right]
\rightarrow
\left[
\begin{array}{c|c}
I_p & 0 \\
\hline
W_{qp} & W_{qq}
\end{array}
\right]
\label{eqn:modified-weight-matrix}
\end{equation}
and then apply the power iteration to update the embedding matrix: $Y^{(t+1)} = WY^{(t)}$. 
 Note that changing the upper $p$ rows to identity does not change the spectral radius of $W$. The observation is derived from the characteristic equation of $W$. In fact, there will be $p$ eigenvectors associated with eigenvalue $1$ after this step,
 guaranteeing that the final embedding matrix is filled with linear combinations of  $W$'s dominant eigenvectors whose weights are determined  by the initial coefficients $c_i$ that are solely determined by $Y_p$. 
\begin{figure}
\begin{subfigure}{.45\linewidth}
\centering
\includegraphics[width=\linewidth]{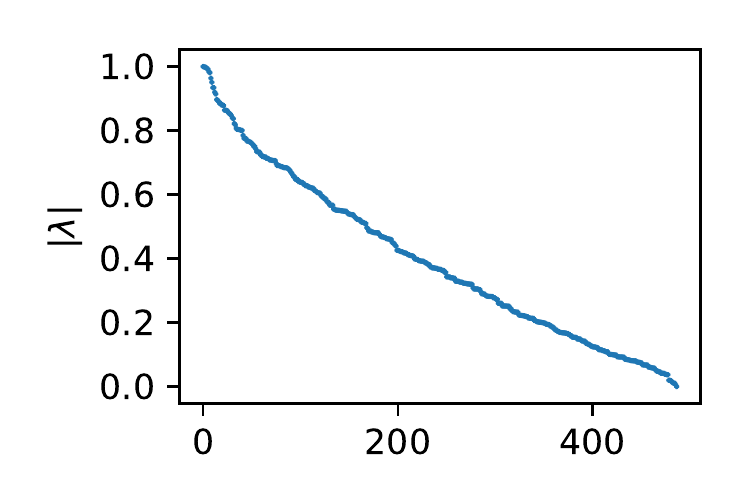}
\caption{}
\label{fig:eigw}
\end{subfigure}%
\begin{subfigure}{.45\linewidth}
\centering
\includegraphics[width=\linewidth]{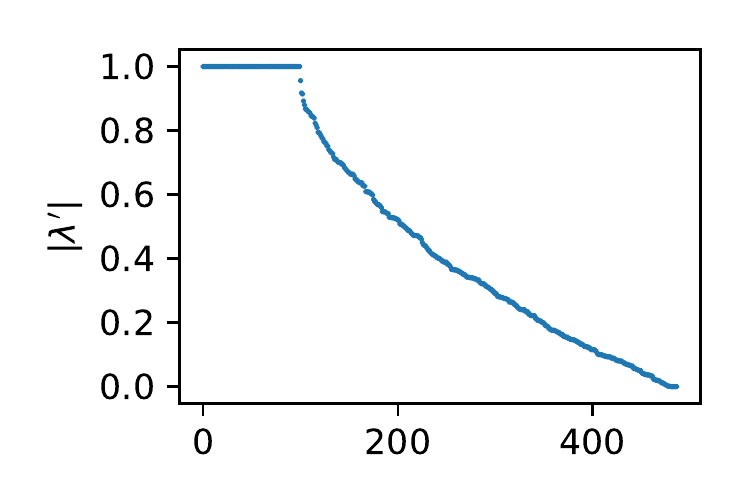}
\caption{}
\label{fig:eigw1}
\end{subfigure}\\
\begin{subfigure}{0.45\linewidth}
\centering
\includegraphics[width=\linewidth]{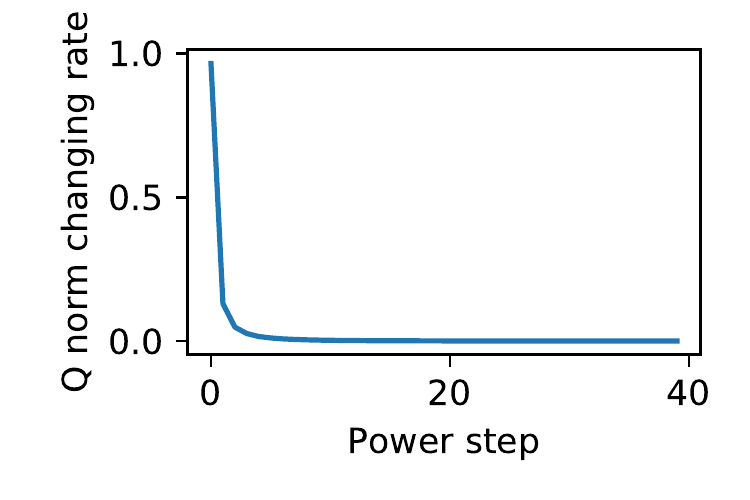}
\caption{}
\label{fig:converge}
\end{subfigure}
 \vspace{-0.05in}
\caption{(a) $|\lambda_i|$ of the original $W$ sorted in a descending order; (b) $|\lambda'_i|$ of $W$ with upper $p$ rows changed to identity (in this toy example, $p=100$, which means there are 100 vectors known and fixed, and there are $p$ eigenvalues that are one); (c) the convergence of the power iteration}
 \vspace{-0.15in}
\label{fig:convergence}
\end{figure}

The stopping criterion is the convergence of $Y_q$ when the $l_1$-norm changing rate of $Y_q$ between two iterations  falls under a predefined threshold $\eta$ or when the maximum number of iterations is reached. 
\begin{equation}
\frac{\left\|Y_q^{(t+1)} - Y_q^{(t)}\right\|_1}{\left\|Y_q^{(t)}\right\|_1} < \eta.
\end{equation}

The computational complexity of the LSI depends on the size of the domain matrix.  The Euclidean distance matrix calculation, $k$-NN search and quadratic programming are embarrassingly paralleled.  It takes less than a minute on a PC to perform LSI on a domain matrix with about 4k samples. The time complexity of building the Euclidean distance matrix is $O(dn^2)$ where $d$ is the dimension of the data space and $n$ is the sample number. The time complexity of building a minimum spanning tree is typically $O(|E|\log{}|V|)$.  It takes approximately $O(|V|^2)$ to  search for additional nearest neighbors for all nodes via partition, where $|V| = n$ (the sample number is equal to the number of vertices in the graph).  LSI only takes a marginal time compared to the general word embedding algorithms that usually require hours or even days for training. 
\begin{algorithm}[]
    \SetKwFunction{Union}{Union}\SetKwFunction{FindCompress}{FindCompress}\SetKwInOut{Input}{input}\SetKwInOut{Output}{output}
    
    \Input{($X$, $\delta$, $Y_p$, $\eta$)} 
    \Output{$Y$} 
    \BlankLine  \tcp*{$Y$: full embedding matrix; $\eta$: stopping criterion}
    initialize $Y_q$; \\
    $\mathcal{G}=(V,E)\gets MSTkNN(X, \delta)$;   \\
    $W \gets NNLS(\mathcal{G}, X)$ \tcp*{Non-Negative Least Squares}
    \For{$i\gets1$ \KwTo $|V|$}{
        $\textbf{w}_{i} \gets \frac{\textbf{w}_{i}}{\Sigma_j w_{ij}}$ \tcp*{normalize weights}
    }
    $W_p \gets I_{|V|p}$; \\
    \While{$\frac{\|Y_q^{(t+1)} - Y_q^{(t)}\|_1}{\|Y_q^{(t)}\|_1} \geq \eta$}{ 
        $Y \gets WY$  \tcp*{power iteration} 
    }
    \caption{Latent Semantic Imputation}
\end{algorithm}

\subsection{Deterministic Convergence of LSI}
To better illustrate the convergence properties of LSI, we follow the logic in~\cite{zhu2002learning}.
From the graph perspective, $I_p$ in Eqn \ref{eqn:modified-weight-matrix} means that the nodes in $\mathcal{G}_p$ have self-loops ($Y_p$ is fixed);  the upper right zero matrix indicates that there are no edges from $\mathcal{G}_q$ to $\mathcal{G}_p$ ($Y_q$ has no influence on $Y_p$); $W_{qp}$ represents the weighted edges from $\mathcal{G}_p$ to $\mathcal{G}_q$ (diffusion from $Y_p$ to $Y_q$); and $W_{qq}$ represents the weighted edges within $\mathcal{G}_q$ (diffusion within $Y_q$). We show that LSI guarantees deterministic convergence regardless of the initialization of $Y_q$.
\begin{theorem}\label{convergent}
If $W_{qq}$ is a convergent matrix, i.e., $\lim_{t \to \infty} W_{qq}^t=0$, then LSI guarantees deterministic convergence. 
\end{theorem}

\begin{proof}
Rewrite the power iteration as follows:
\begin{align}
  & \left[Y_p^{(t+1)}, Y_q^{(t+1)}\right] = \left[Y_P^{(t)}, W_{qp}Y_p^{(t)}+W_{qq}Y_q^{(t)}\right]. \\
 & \lim_{t \to \infty}Y_q^{(t)} = \lim_{t \to \infty}W_{qq}^tY_q^{(0)} + \left[\sum_{i=0}^{t-1}W_{qq}^{i-1} \right]W_{qp}Y_p.  
\end{align}

Given $\lim_{t \to \infty} W_{qq}^t=0$, the final embedding result is deterministic regardless of $Y_q^{(0)}$. 
\end{proof}

Now we show that under our algorithm settings, $W_{qq}$ is a convergent matrix. Note that $W_{qq}$ is a substochastic matrix, \textit{i.e.}, \\
$\exists i, \sum_{j} {(W_{qq})}_{ij} < 1\ and\ \forall i, 0\leq \sum_j {(W_{qq})}_{ij} \leq 1$.  This is a result of the minimum spanning tree where there exists one edge from $\mathcal{G}_p$ to $\mathcal{G}_q$ (one positive value in $W_{qp}$), and hence, there exists one row sum of $W_{qq}$ strictly less than $1$. In fact, under the MST-$k$-NN graph setting, there are many edges from $\mathcal{G}_p$ to $\mathcal{G}_q$ due to the additional nearest neighbor search. After all, these edges serve the major assumption of imputation and allow us to leverage the known embeddings for diffusion. Additionally, due to the minimum spanning tree step at the beginning, we have the following lemma: 

\begin{lemma}\label{connectivity}
For every node in $\mathcal{G}_q$, there always exists a path from $\mathcal{G}_p$ to this node.
\end{lemma}
\begin{definition}[\textbf{Sink Node}]
Let $r_i = \sum_j(W_{qq})_{ij}$, the $i$-th row sum. A sink node in a substochastic matrix is one with $r_i < 1$.
\end{definition}
Given Lemma~\ref{connectivity}, we have the following corollary:
\begin{corollary}
For every node in $\mathcal{G}_q$, either it is a sink node or there exists a path from a sink node to it, or both.
\end{corollary} 

\begin{lemma}\label{substochastic}
For a substochastic matrix, for every non-sink node, if there exists a path from a sink node to this non-sink node, then the substochastic matrix is  convergent. 
\end{lemma}
\begin{proof}
To show $\lim_{t \to \infty} W_{qq}^t=0$, we need to show 
$$\forall i, \lim_{t \to \infty} r_i^{(t)} =  \lim_{t \to \infty} \sum_{j=1}^q \left(W_{qq}^{(t)}\right)_{ij} = 0,$$
or $\forall i,$ for a finite $t$
$$ \sum_{j=1}^q \left(W_{qq}^{(t)}\right)_{ij}  < 1.$$
 For every sink node $v_{k^*}$ in $\mathcal{G}_q$, we have $r_{k^*} < 1$. And $\forall t > 1,$
\begin{eqnarray*}
r_{k^*}^{(t)} &=& \sum_{k=1}^q\sum_{j=1}^q(W_{qq})_{k^*j} \left(W_{qq}^{(t-1)}\right)_{jk} \\
& = & \sum_{j=1}^q(W_{qq})_{k^*j}\sum_{k=1}^q\left(W_{qq}^{(t-1)}\right)_{jk} = \sum_{j=1}^q(W_{qq})_{k^*j}r_j^{(t-1)}. 
\end{eqnarray*}
Since we have $\forall i, \forall t>0, r_i^{(t)}\leq 1$, 
$$
r_{k^*}^{(t)} = \sum_{j=1}^q (W_{qq})_{k^*j}r_j^{(t-1)} \leq \sum_{j=1}^q (W_{qq})_{k^*j} = r_{k^*} < 1. 
$$
Thus, the convergence is apparently true for those sink nodes. Suppose the shortest path (with all positive edges) from a sink node $v_{k^*}$ to a non-sink node $v_i$ within $\mathcal{G}_q$ has $m$ steps. Then, we have 
$$
\left(W_{qq}^{(m)}\right)_{ik^*} > 0
$$
and 
$$
r_{k^*} < 1. 
$$
Hence, the following condition holds: 
$$r_i^{(m+1)} = \sum_{j=1}^q \left(W_{qq}^{(m)}\right)_{ij} r_j < \sum_{j=1}^q \left(W_{qq}^{(m)}\right)_{ij} = r_i^{(m)} \leq 1, i\neq k^*.$$
Because our graphs are always finite, the convergence also holds for the non-sink nodes. 
\end{proof}

Combining Lemma~\ref{connectivity} and \ref{substochastic}, we conclude that  $W_{qq}$ is a convergent matrix  under our algorithm settings. Hence, LSI guarantees a deterministic convergence. This property does not always hold for a $k$-NN graph. This is why we start with a minimum spanning tree in the graph construction.

\section{Experiments} 

We applied two types of  evaluation methods for word embedding: intrinsic evaluation and extrinsic evaluation~\citep{schnabel2015evaluation}, to validate LSI. Categorization is an intrinsic evaluation in which word embedding vectors are clustered into different categories. Clustering is also one of the principles for designing the hierarchical softmax word embedding model~\citep{mnih2009scalable}.  The extrinsic evaluation involves downstream language tasks. In our experiments, we used the $k$-NN classifier on a group of embedding vectors and inspect the accuracy. We also fed the embeddings as input to long short-term memory (LSTM) networks to investigate the performance in terms of  the  language modeling perplexity. 
 

In the first experiment,  we used word2vec~\citep{mikolov2013distributed} to train word embeddings on the Wiki corpus.   In this dataset,  each company has an industry category label, \textit{e.g.} Google belongs to the IT industry, while Blackrock belongs to the financial industry, and these labels are used in the $k$-Nearest-Neighbors classification.  We conjecture that if the latent semantic imputation works efficiently and the extra knowledge is successfully transferred from the affinity space into the semantic space, the $k$-NN classification accuracy on the imputed embedding matrix will be higher than that of the original embedding matrix purely based on the training corpus. Furthermore, we  conducted experiments based on Google word2vec~\citep{mikolov2013distributed}, Glove~\citep{pennington2014glove}, and fastText~\citep{bojanowski2016enriching} pretrained word embeddings for comparison. 
This is a multi-class classification problem with a total of eleven categories representing eleven industry sectors. 


The second experiment is on language modeling. Many efforts have shown that a reliable word embedding boosts the performance of  neural-network-based language tasks~\citep{lai2015recurrent,kim2014convolutional,ganguly2015word}. 
Indeed, the embedding vectors are essentially the weight matrix of the neural network's first hidden layer~\citep{bengio2003neural,mikolov2013linguistic}.
We used the imputed embeddings as the input to LSTM~\citep{hochreiter1997long,sundermeyer2012lstm} to further evaluate the efficacy of LSI enhancing the word embeddings.  

\subsection{Classification on Embedding Vectors}
We gathered the Wikipedia URLs of the S\&P500 companies and then their associated external links in these Wiki pages.   These linked Wiki pages are directly relevant to the S\&P500 companies.  We downloaded about 50k wiki pages and treated them as the corpus for training the word embeddings.  We removed the punctuation marks and stopwords from the corpus, converted all upper-case letters to lower case,  and tokenized n-gram company names in these wiki pages. We followed the TensorFlow word2vec~\citep{mikolov2013efficient} example with some minor changes to the hyperparameters. We set the learning rate to 0.05, the total number of steps to five million (when evaluation words have meaningful neighbors) and the minimal word occurrence threshold to 40. All other hyperparameter values remain the same as those in the original example. Since the S\&P500 components evolve, we fixed the duration of the time series  on the historical stock price to construct the domain matrix and finally chose 487 long-term stable companies.   During the pre-processing, 384 company names had a frequency greater than 40, and the rest were excluded from the vocabulary because of their low frequency. The vocabulary size is 50,000. We identified another vocabulary that is a combination of the vocabulary we just described and the remaining  company names with low frequency for experimental comparison. 

The  word2vec embedding in our experiments is pre-trained from 100 billion words, and there are  300 dimensions and 3 million words in the vocabulary. The Glove pre-trained embedding is based on 840 billion words and has 300 dimensions and a vocabulary of 2 million words. The fastText pre-trained embedding is based on 600 billion words and has 300 dimensions with a vocabulary size of 2 million.
\begin{figure}
\begin{subfigure}{.5\linewidth}
\centering
\includegraphics[width=\linewidth]{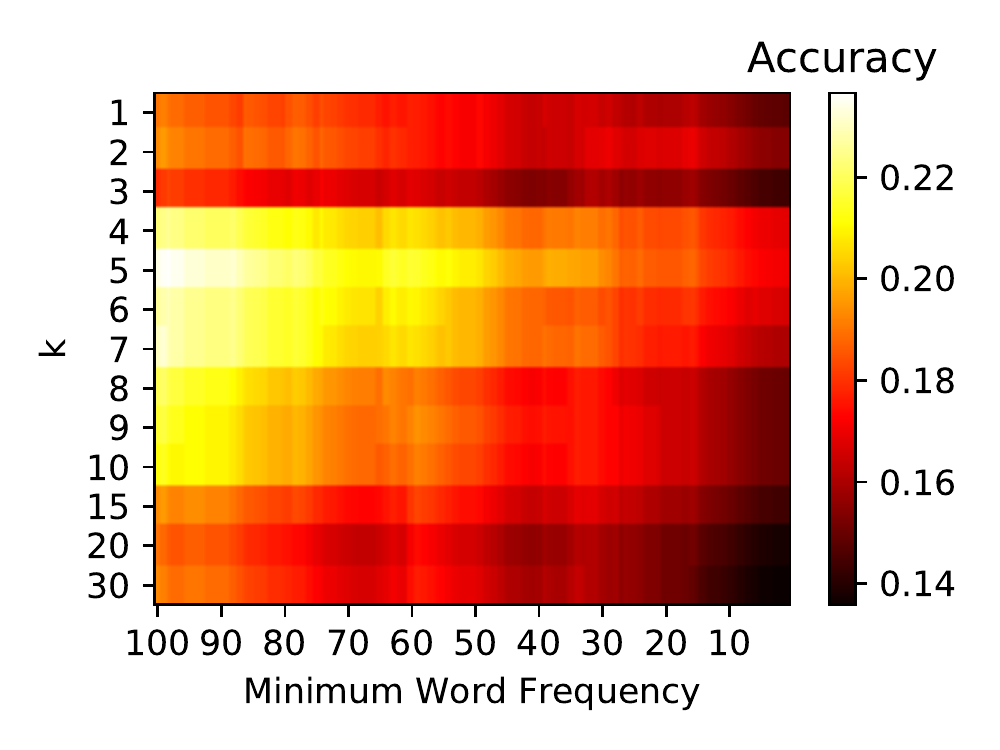}
\caption{}
\label{fig:heat_before}
\end{subfigure}%
\begin{subfigure}{.5\linewidth}
\centering
\includegraphics[width=\linewidth]{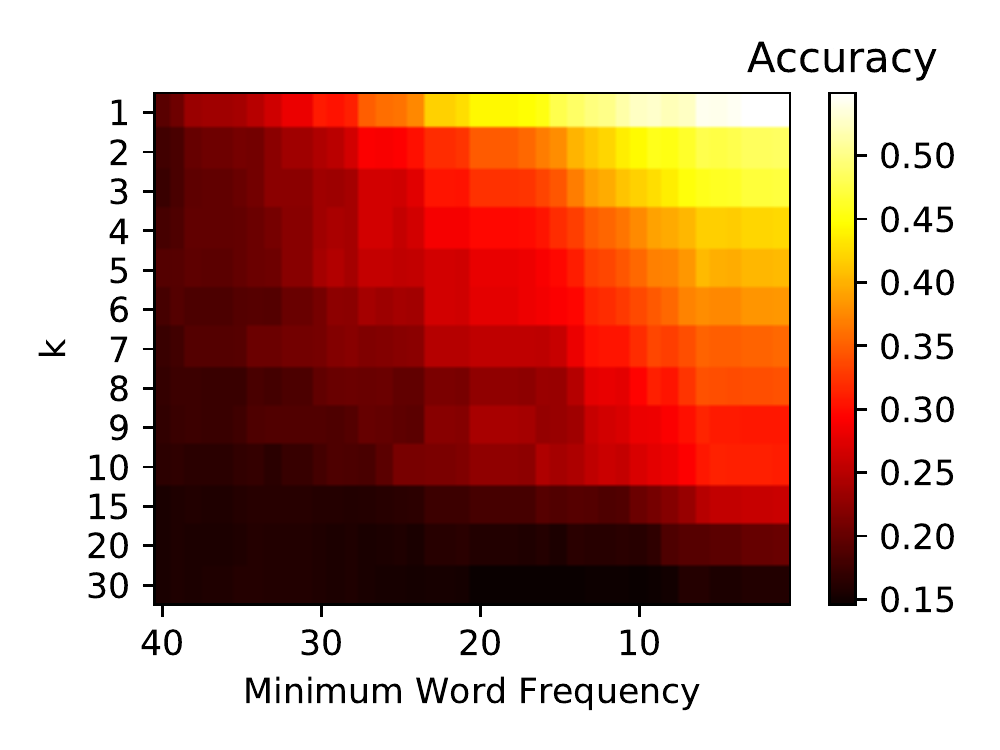}
\caption{}
\label{fig:heat_after}
\end{subfigure}
\caption{$k$-NN accuracy on word embedding vectors. (a) Self-trained embedding on Wiki corpus (b) Self-trained embedding on the same Wiki corpus combined with the affinity knowledge via LSI.}
\label{fig:heat}
\end{figure}


To transfer the external knowledge to the semantic space, we gathered the historical daily stock returns  of the S\&P500 component companies from 2016-08-24 to 2018-08-27 and used the correlation matrix as the affinity matrix (domain matrix). The minimal degree of the graph $\delta$ is set to be 8 and the stopping criterion $\eta$ is set to be $1e^{-2}$ in the initial experiment based on our observation  that the algorithm is relatively robust to the values of $\delta$ and $\eta$. 
\begin{table}
  \caption{$k$-NN accuracy (\%) on embedding vectors}
  \scalebox{0.9}{
  \begin{tabular}{c | c c c c c c c c }
    \toprule
    \diagbox[]{$E$}{$k$} & 2 & 5 & 8 & 10 & 15 & 20 & 30\\
    \midrule
    self & 0.154 & 0.170 & 0.150 & 0.150 & 0.144 & 0.138 & 0.135 \\
    self(hf) & 0.180 & 0.190 & 0.172 & 0.167 & 0.157 & 0.157 & 0.157 \\
    \textbf{self(hf)+aff} & 0.556 & 0.472 & 0.396 & 0.359 & 0.302 & 0.261 & 0.187 \\
    Google & 0.220 & 0.297 & 0.271 & 0.305 & 0.280 & 0.280 & 0.186\\
    \textbf{Google+aff} & 0.838 & 0.803 & 0.784 & 0.768 & 0.725 & 0.678 & 0.626 \\
    Glove & 0.417 & 0.466 & 0.490 & 0.500 & 0.500 & 0.505 & 0.451 \\
    \textbf{Glove+aff} & 0.832 & 0.766 & 0.690 & 0.653 & 0.606 & 0.542 & 0.405 \\
    fast & 0.443 & 0.496 & 0.527 & 0.500 & 0.511 & 0.470 & 0.447 \\
    \textbf{fast+aff} & 0.811 & 0.749 & 0.713 & 0.684 & 0.641 & 0.608 & 0.595\\
    \bottomrule
  \end{tabular}}
  \label{tab:freq}
\end{table}

Figure~\ref{fig:heat_before} shows a clear pattern that as the minimum word frequency decreases, the $k$-NN accuracy decreases no matter how $k$ changes.  The pattern indicates that the low-frequency words are associated with the embedding vectors with low quality.  Table~\ref{tab:freq}  also shows that the self-embedding without low-frequency words (self(hf)) yields better performance.  From Figure~\ref{fig:heat_after}, we observed that the accuracy of $k$-NN  improve once the LSI is introduced, \textit{i.e.}, the embedding with LSI outperforms the vanilla embedding purely based on the training corpus by a large margin.  The $k$-NN accuracy also demonstrated improvement when we applied LSI to impute the subset of company names in the word2vec, Glove, and fastTest pre-trained embeddings. Such an improvement is evident in the low-frequency words. 

To verify whether LSI is robust to its hyper-parameter $\delta$ and stopping criterion $\eta$, we did multiple investigations.  When we set $\eta = 1e^{-2}$ and let $\delta$ vary, we observed that LSI is relatively robust to a varying $\delta$ under the constraint that $\delta$ is not too large in which case the manifold assumption is significantly violated, or too small, which causes one or two neighbors to dominate.  When we set the minimum degree of the graph $\delta=8$ and let $\eta$ vary, we also had the same observation that the LSI is robust to the stopping criterion $\eta$. 
\begin{figure}
\begin{subfigure}{.5\linewidth}
\centering
\includegraphics[width=\linewidth]{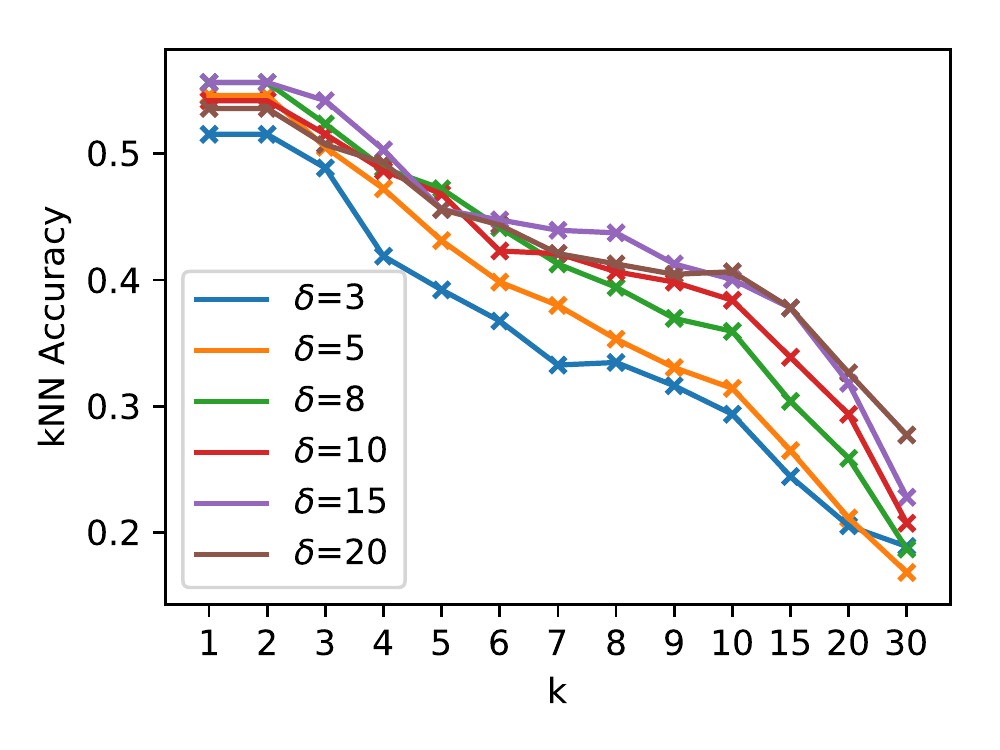}
\caption{Sensitivity of $\delta$}
\label{fig:delta sens}
\end{subfigure}%
\begin{subfigure}{.5\linewidth}
\centering
\includegraphics[width=\linewidth]{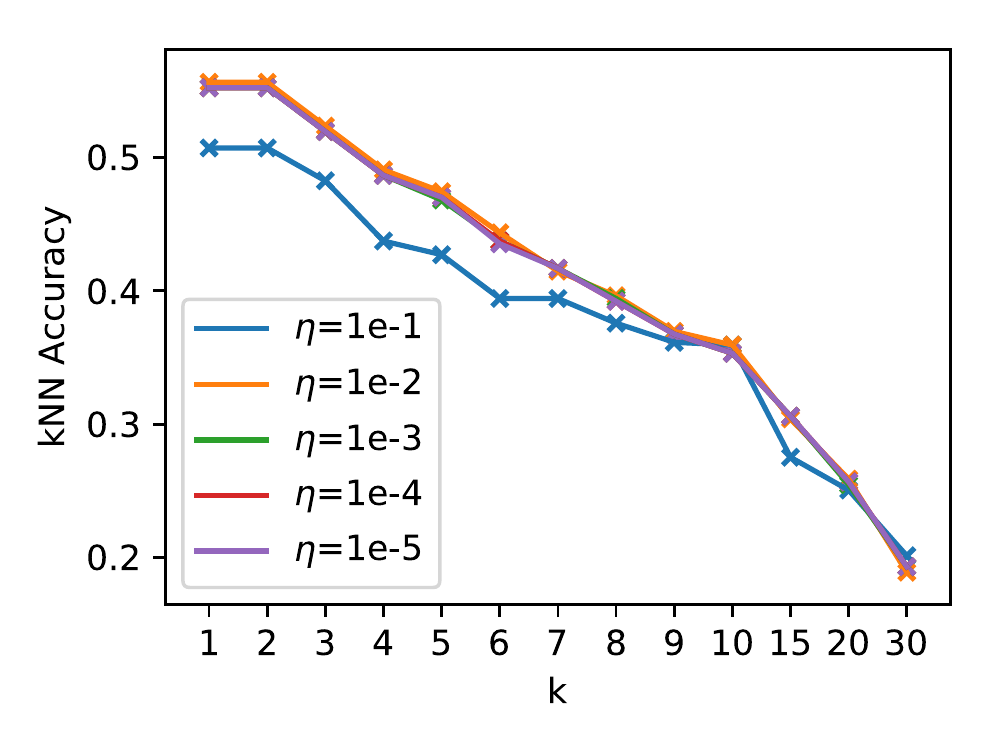}
\caption{Sensitivity of $\eta$}
\label{fig:eta sens}
\end{subfigure}
\caption{Sensitivity Tests}
\label{fig:sensitivity test}
\end{figure}





\subsection{Language Modeling on News Headlines}
The straightforward way to measure the efficacy of different word embeddings is to feed them into a language model and inspect the perplexity. Language models are fundamental for many language tasks, including speech recognition, machine translation, parsing, and handwriting recognition. The goal of language modeling is to find the next most likely word following a word sequence, \textit{i.e.}, to maximize the conditional probability $P(w_t|w_1, w_2, ..., w_{t-1})$.  Embedding words into a continuous Euclidean space~\citep{bengio2003neural} is essential for the performance of language models.

We retrieved an experiment dataset from the Wharton Research Data Services (WRDS) and split the dataset into three folds: the training fold, which contains about 37.5k financial news headlines, and the validation and testing folds, each of which contains about 8k headlines. The corpus includes several thousand company names that are publicly listed on NYSE and NASDAQ.  The dataset is small and it is difficult to attain reliable word embedding from the corpus directly. We adopted the three embeddings pre-trained on large corpora that also appeared in the previous experiment,  used them as the initial embeddings to be enriched with domain knowledge,  and recovered the embedding vectors for those missing company names via LSI.   

In the pre-processing,  we replaced the tokens with a frequency smaller than 3 with $\langle unk\rangle$ and used $\langle eos\rangle$ to denote \textit{End of Sentence}.  Those tokens that appear in the validation and testing folds and not in the training fold are also replaced with $\langle unk\rangle$.  The historical daily stock price data covers thousands of companies listed on NYSE and NASDAQ. We took the time series for 400 trading days ending on November 1, 2018, and removed those companies with more than 20\% missing values. We used the daily return of these companies and filled the missing values with the mean of each day's company return.  The final data for the domain matrix contain 4092 companies, of which more than 3000 of these company names are not in the pre-trained embeddings. 

We also trained an embedding matrix from scratch using the TensorFlow word2vec example code on the training set for language modeling. 
The learning rate is set to 0.1, the number of steps to two million, the minimum word frequency to 10, and the  $\sigma$ in the Gaussian initialization to 0.1. We name the obtained embedding ``self-embedding'' and use it as a baseline. The vocabulary for all embedding matrices is the same.   For those missing vectors that can not be recovered by LSI,  we randomly initialize them according to a Gaussian distribution with $\sigma=0.3$ (the standard deviations of the word2vec, Glove, and fastText pre-trained embeddings along one dimension are 0.16, 0.38 and 0.25  respectively), and these randomly initialized word vectors will be updated by the neural net during training.

Our language model follows the TensorFlow benchmark for building a  LSTM model on the Penn Bank Tree (PTB) dataset (small model settings) with some minor changes to the hyperparameters. We allow the model to use the pre-trained embedding and let it update the embedding vectors during training.   We set the initial learning rate to 0.5 and let it decrease with a decay rate of 0.5 after two epochs.  We found that the validation perplexity stopped declining and instead climbed up gradually at the end of epoch five, and thereby, we did an early stopping on the training. The language model with each embedding is trained  ten times, and the result in Table~\ref{pp} is the average decrease in perplexity.  
\begin{table}
  \caption{Language Model Perplexity of Different Embeddings}
  \begin{tabular}{c|cc}
    \toprule
    Embedding & Test PP & \%decrease\\
    \midrule
    self & 13.093 &  \\
    \textbf{self+Google} & 12.742 & 2.75 \\
    \textbf{self+fastText} & 12.477 & 4.94 \\
    \textbf{self+Glove} & 12.646 & 3.54 \\
    Google & 12.431 & 5.33 \\
    fastText & 12.215 & 7.19 \\
    Glove & 12.218 & 7.16 \\
    \textbf{self+aff} & 11.883 & 10.18 \\
    \textbf{Google+aff} & 11.646 & 12.42 \\
    \textbf{fastText+aff} & 11.638 & 12.51 \\
    \textbf{Glove+aff} & 11.510 & 13.76 \\
    \bottomrule
  \end{tabular}
  \label{pp}
\end{table}

\begin{figure*}[ht]
  \centering
  \includegraphics[width=\textwidth]{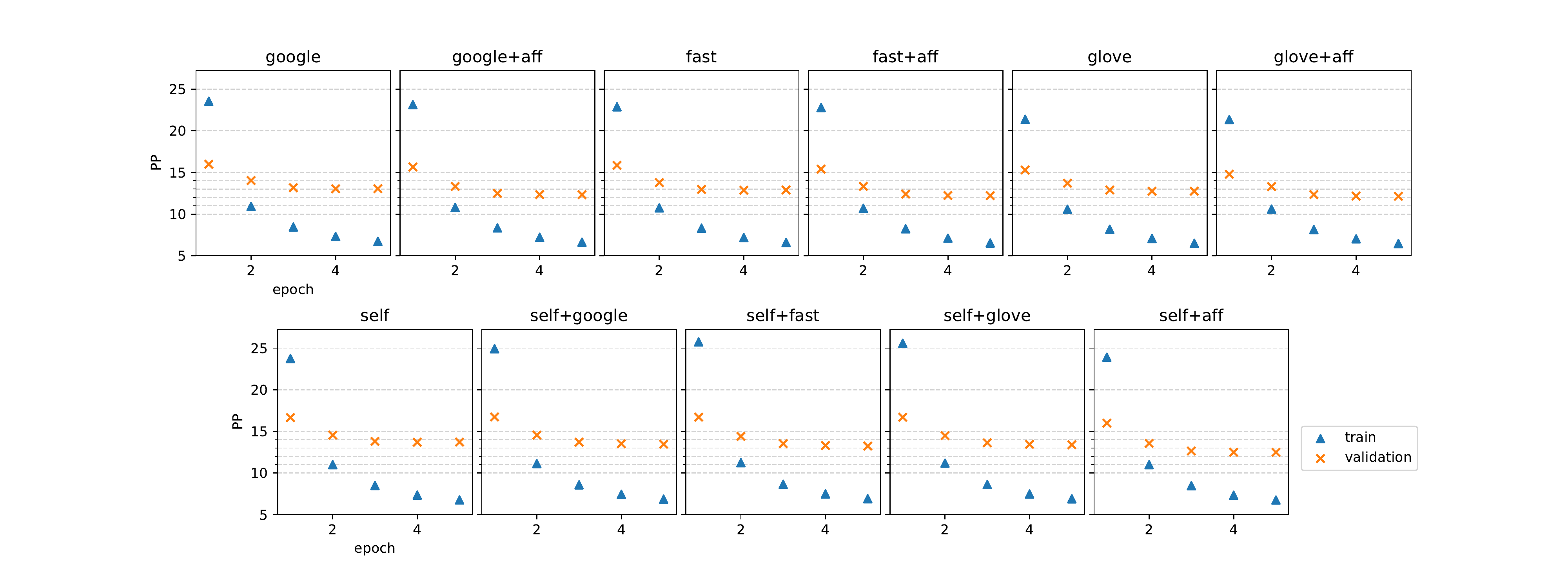}
  \vspace{-0.3in}
  \caption{Training and Validation Perplexity with Different Embeddings}
 \vspace{-0.15in}
\end{figure*}

The experiment results show that the LSTM with self-embedding has the largest testing perplexity.  In Table~\ref{pp}, we observed that the self-embedding plus general-purpose pre-trained embeddings via LSI leads to a 2-5\% perplexity decrease.  The observation demonstrates that when the corpus is small, it is difficult to obtain a reliable word embedding.  In contrast,  the pre-trained embedding on a vast corpus, such as Google-news with 100 billion words, is a better choice. However, the pre-trained embedding model on a large corpus is for the general purpose and will skip many keywords that are crucial to  domain-specific language tasks.  To obtain a  more reliable domain-aware word embedding,  we need to identify a domain-specific matrix and inject it into the general-purpose embedding.  The Google-news (w2v), Glove and fastText pre-trained embeddings plus the domain matrix via LSI obtained an about 5\% decrease in perplexity  compared with those without LSI and a more than 12\% decrease compared with the self-trained embedding.

\section{Discussion and Conclusions}
In this paper, we present a novel approach, Latent Semantic Imputation, for enhancing domain word embedding. It is capable of  transferring prior knowledge from a domain matrix to the semantic space.  LSI provides a novel approach to unifying the entity representations defined in different Euclidean spaces and has three major steps: (1) construct a Minimum-Spanning-Tree $k$-Nearest-Neighbor graph from a domain matrix, (2) solve for the optimal weights using Non-Negative Least Squares,  and (3) recover the missing embedding vectors in the semantic space using the power iteration method.  LSI has only two hyper-parameters, the minimal degree of the graph $\delta$ and the stopping criterion  $\eta$ in the power iteration step, and is robust to different values of the hyper-parameters. We also rigorously analyzed the deterministic convergence of LSI.  Latent Semantic Imputation can be used for solving the problem of unreliable word embedding when the size of a training corpus is limited or when the critical domain words are missing from the vocabulary.  Our experiments demonstrated how the frequency of a word negatively impacts the quality of word embedding and how LSI mitigates the impact. Specifically, the downstream language modeling with imputed embedding yields better performance than does the benchmark model pre-trained from large corpora. 

\section{Acknowledgements}
Heartfelt thanks to Yiannis Koutis of NJIT  and  Xiangmin Jiao of Stony Brook University for their helpful discussions and to the anonymous reviewers for their comments. 
\bibliographystyle{ACM-Reference-Format}
\bibliography{sample-base}

\end{document}